\documentclass[twocolumn]{IEEEtran}

\usepackage{tikz}
\usetikzlibrary{bayesnet}
\usetikzlibrary{arrows}

\usepackage{hyperref}

\usepackage{url}
\usepackage{color}

\newcommand{\PP}{{\mathbb{P}}}
\usepackage{amsmath} 
\usepackage{amsthm}
\usepackage{caption}
\usepackage{subcaption}
\usepackage{booktabs} 

\newtheorem{lemma}{Lemma}

\usepackage{cite}
\usepackage{amsmath,amssymb,amsfonts}
\usepackage{algorithmic}
\usepackage{graphicx}
\usepackage{textcomp}
\usepackage{xcolor}

\newcommand{\y}{\boldsymbol{y}}

\newcommand{\EE}[1]{{\mathbb{E}}\left[#1\right]}
\newcommand{\EEx}[1]{{\mathbb{E}_x}\left[#1\right]}

\newcommand{\norm}[1]{\left\|#1\right\|}

\begin{document}
\title{Conditional Frechet Inception Distance}

\author{Michael Soloveitchik, Tzvi Diskin, Efrat Morin and Ami Wiesel\thanks{The authors are with The Hebrew University of Jerusalem, Israel . This research was funded by Center for Interdisciplinary Data Science (CIDR) in the Hebrew University, Israel.}} 

\maketitle

\begin{abstract}
We consider distance functions between conditional distributions. We focus on the Wasserstein metric and its Gaussian case known as the Frechet Inception Distance (FID). We develop conditional versions of these metrics, analyze their relations and provide a closed form solution to the conditional FID (CFID) metric. We numerically compare the metrics in the context of performance evaluation of modern conditional generative models. Our results show the advantages of CFID compared to the classical FID and mean squared error (MSE) measures. In contrast to FID, CFID is useful in identifying failures where realistic outputs which are not related to their inputs are generated. On the other hand, compared to MSE, CFID is useful in identifying failures where a single realistic output is generated even though there is a diverse set of equally probable outputs.
\end{abstract}

\section{Introduction}
 Generative models have revolutionized the machine learning world by learning how to efficiently sample from an unknown distribution. Such models have become ubiquitous	since the introduction of the seminal Generative Adversarial Network (GAN) \cite{goodfellow2014generative}. Conditional GANs  were then developed to sample from a conditional distribution \cite{mirza2014conditional}. Recent unsupervised training of models can even learn from an unpaired dataset when the inputs and outputs are not aligned \cite{zhu2017unpaired,richardson2020surprising,huang2018multimodal,armanious2019unsupervised}. Every day new models are developed improving state of the art, such as speech-to-image translations \cite{li2020direct}. Yet, it is still poorly understood how to measure their performance. There is a large body of works on unconditional metrics that help in comparing different models \cite{lucic2017gans,chien2019variational}. Among these, the Frechet Inception Distance \cite{heusel2017gans} (FID) has become a popular measure due to its simplicity. Perhaps surprisingly, it is also frequently used in the analysis of conditional generators, e.g., \cite{ravuri2019classification,zhou2019hype,li2020direct}. The goal of this work is therefore to introduce conditional versions of FID, develop the necessary machinery and examine it in different models and datasets. 


Measuring the distance between two distributions is a fundamental problem in statistics. There are classical metrics based on likelihoods, divergences and other statistical notions, e.g. Kullback Leibler divergence (KLD) \cite{rubner2000earth}, Wasserstein distance (WD) \cite{villani2008optimal}, and total variation \cite{villani2008optimal}. Each emphasizes various properties of the distributions with different sample and computational complexities.
Unfortunately, modern applications involve huge dimensional datasets in which estimating these distances is difficult. Practical approaches consider simplified measures based on low dimensional embeddings, accuracy in some downstream task and/or some restricted parametric family. Arguably the most widely adopted metric in the computer vision literature is the Inception Score (IS) \cite{salimans2016improved}. It is applicable for classification tasks and is based on the KLD of the labels. Another measure is based on ``inference through optimization'' where a nearest GAN sample is optimized per real sample \cite{metz2016unrolled}. It is also popular to use human evaluation, specifically the ``real vs fake'' test, in which a pair of fake and real images is shown to a user and she needs to identify the real image \cite{zhang2016colorful}. Many metrics are designed to evaluate image quality. For example, LPIPS use learned features to measure the perceptual difference between two images \cite{zhang2018unreasonable}. SSIM uses correlations between the pixels of the images \cite{wang2004image}. PSNR measures the peak signal to noise ratio between two monochrome images. Finally, one of the simplest and most popular metrics is FID which is based on an InceptionV3 embedding, and a particularization of WD to the simple case of multivariate normal distributions \cite{heusel2017gans}.

Developing accurate and simple distances between conditional distributions is a more challenging problem. Instead of having two distributions, we now have two classes of distributions conditioned on a common input (feature). Most works simply ignore the conditioning and use unconditional metrics, e.g. \cite{alqahtani2019analysis}, or combine FID measure with some other metric such as SSIM \cite{you2019pi}.
When the conditioning variables are discrete valued, it is natural and simple to resort to multiple unconditional metrics for each possible value \cite{murray2019pfagan}. This is clearly intractable in the continuous case. Instead, we take the natural next step and consider the average distance between  conditional distributions. We begin with the general Wasserstein case and then focus on the special Gaussian FID case.
Together, the  main contributions  of  this  paper are:
 \begin{itemize}
\item We define the conditional Wasserstein distance (CWD) and provide a closed form conditional FID (CFID) in the Gaussian case.
\item We analyze the relations between the different metrics. We prove that the MSE bounds CWD from above, whereas WD bounds it from below. We prove that the conditional metric can be interpreted as a joint metric that emphasizes the conditioning variable.
\item We demonstrate the advantages of CFID over competing metrics in evaluating modern conditional generative models. On one hand, CFID is better than FID in identifying failures where realistic outputs which are not related to their inputs are generated. On the other hand, CFID is better than MSE in identifying failures where a single realistic output is generated even though there is a diverse set of equally probable outputs.
\end{itemize}

\subsection{Related Work}
\subsubsection{Image to Image Translation}
CFID is aimed at evaluating conditional generators and in particular image-to-image translation models. These involve different tasks including image colorization \cite{zhang2016colorful}, image super-resolution \cite{ledig2017photo} and different domain transfer tasks such as edges to natural image or map to aerial image \cite{isola2017image}. 

Most of the above tasks are ill-conditioned and significant effort is spent on learning the full conditional probability rather than a single solution in the new domain. For this purpose, BicycleGAN \cite{zhu2017toward} uses a latent vector that can be sampled at test time.

Image-to-image translation models can be divided into paired (supervised) models and unpaired (unsupervised) models. The paired models \cite{isola2017image,zhu2017toward} are trained with an aligned dataset of inputs and outputs, whereas unpaired models rely on decoupled datasets and must resort to different techniques to enforce the relation between input and the output. CycleGAN \cite{Cyclegan2017} uses cycle consistency, UNIT \cite{liu2017unsupervised} adds weight sharing constraints and \cite{richardson2020surprising} uses a linear mapping together with an iterative closest points algorithm. 

Finally, MUNIT \cite{huang2018munit} allows unpaired training with diverse outputs. 
 Augmented CycleGAN \cite{almahairi2018augmented} allows a flexible portion of paired data. Our results show that CFID is particularly useful in analyzing these advanced unpaired yet diverse models.

\subsubsection{Variants of FID}
FID is a standard measure for GAN evaluation and there are lots of variations and extensions. 

The original FID relied on an IncectionV3 embedding \cite{szegedy2016rethinking}. Other representations were proposed for different tasks: Frechet ChemNet distance for molecules imagery \cite{preuer2018frechet} or Frechet video distance for video \cite{unterthiner2018towards}. Recently, \cite{morozov2020self} showed that features based on a self-supervised trained model are preferable in natural images. 

 Class-aware FID \cite{liu2018improved} is useful when the dataset consists of discrete classes which can be evaluated separately. Other conditional FID metrics for discrete inputs include \cite{miyato2018cgans,murray2019pfagan,benny2021evaluation}. 
 
 The closest work to CFID is \cite{devries2019evaluation} where a version of FID for joint distributions of continuous inputs and outputs was considered. In Section 4, we elaborate on the relations between CFID and this joint metric. Briefly, we show that CFID can be interpreted as a joint metric in which we restrict the attention to identical input distributions and emphasize their importance by scaling the inputs without bound.
 
 Finally, FID was originally derived for Gaussian distributions and has recently been extended to more expressive Gaussian mixtures \cite{assa2018wasserstein,dimitrakopoulos2020wind}.


\section{Conditional Wasserstein}\label{sec_cw}

\begin{figure}
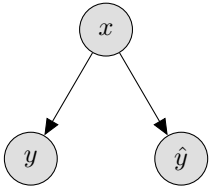

         \centering
         \tikz{
 \node[obs] (y) {$y$};%
 \node[obs,xshift=2cm] (haty) {$\hat y$};%
 \node[obs,above=of y,xshift=1cm] (x) {$x$}; %
 \edge {x} {y};
 \edge {x} {haty} }
     
        \caption{Graphical Model for conditional model.}
    \label{fig:cond_graphical_model}
\end{figure}

In this section, we define the conditional Wasserstein distance between two conditional distributions. For this purpose consider a triplet of random vectors $\{\hat{y},y,x\}$. 
Both  $y$ and $\hat y$ are conditioned on the same $x$ as illustrated in Fig. \ref{fig:cond_graphical_model}. Our goal is to develop a distance function between the distributions ${\mathbb{Q}}_{{y}|x}$ and ${\mathbb{Q}}_{\hat{y}|x}$ conditioned on $x\sim {\mathbb{Q}}_x$. 

The conditional model induces the following distributions
\begin{eqnarray}
  {\mathbb{Q}}_{\hat y,x}&=&  {\mathbb{Q}}_{\hat y|x} {\mathbb{Q}}_{x}\nonumber\\
  {\mathbb{Q}}_{y,x}&=&  {\mathbb{Q}}_{y|x} {\mathbb{Q}}_{x}\nonumber\\
  {\mathbb{Q}}_{y}&=& \int {\mathbb{Q}}_{y,x} dx \nonumber\\
  {\mathbb{Q}}_{\hat y}&=& \int {\mathbb{Q}}_{\hat y,x} dx
\end{eqnarray}
These distributions do not fully characterize the full joint distribution as we have not specified ${\mathbb{Q}}_{y,\hat y|x}$. Typically, it is assumed that $y$ and $\hat y$ are conditionally independent given $x$ but we do not rely on any such assumption in the paper. 

The classical and simplest statistical metric for comparing $y$ and $\hat y$ is the MSE:
\begin{equation}\label{mse}
    {\rm MSE} = \EE{\norm{y - \hat y}^2}
\end{equation}
Due to the total law of expectation, computing the conditional MSE and then averaging is identical to the standard MSE:
\begin{equation}\label{cmse}
    {\rm MSE} =\EEx{\EE{\norm{y - \hat y}^2|x}} 
\end{equation}
Unlike the metrics below, computing the MSE requires additional knowledge on ${\mathbb{Q}}_{y,\hat y|x}$. For completeness, we note that, in the context of image processing, it is less common to use MSE, but simple variations on it known as PSNR or SSIM.

The main downside to MSE is that it compares specific realizations of $y$ and $\hat y$ rather than their distributions. For this purpose, WD is a standard performance measure. It is defined as:
    \begin{equation}
    {\mathcal{WD}}({{\mathbb{Q}}_{y}};{\mathbb{Q}}_{\hat y})=
    \left\{
    \begin{array}{ll} \min_{\PP_{y,\hat y}}  & \mathbb{E}\left[\|y-\hat y\|^2\right] \\
   {\rm{s.t.}} & \PP_{y,\hat y}\in\Pi({\mathbb{Q}}_y;{\mathbb{Q}}_{\hat y})
    \end{array}
    \right.
    \end{equation}

where $\Pi$ is the set of joint distributions with prescribed marginals:
 \begin{equation}
    \Pi({\mathbb{Q}}_y;{\mathbb{Q}}_{\hat y})=\left\{\PP_{y,\hat y} \left|  \int
    \PP_{y,\hat y}dy = {\mathbb{Q}}_{\hat y},\;
    \int\PP_{y,\hat y}d\hat y= {\mathbb{Q}}_{y}\right.\right\}.
    \end{equation}
Remarkably, WD does not need access to the joint distribution of $\{y, \hat y\}$ and relies on the joint distribution that minimizes the MSE.

Based on the above definitions, we are now ready to define CWD.
Given some $x$, one can use ${\mathcal{WD}}({\mathbb{Q}}_{y|x};{\mathbb{Q}}_{\hat y|x})$. But this will result in a function of $x$ which is a random variable. Therefore, similarly to (\ref{cmse}), we define CWD by taking an outer expectation with respect to $x$:
   \begin{eqnarray}
    {\rm{CWD}}=\mathbb{E}_x\left[{\mathcal{WD}}({{\mathbb{Q}}_{y|x}};{\mathbb{Q}}_{\hat y|x})\right]
    \end{eqnarray}


To gain more intuition on CWD and its properties, it in instructive to order the three metrics.
\begin{lemma}\label{wd_relations}
   \begin{equation}
       \rm{MSE} \geq {\rm{CWD}} \geq {\rm{WD}}.
   \end{equation}
\end{lemma}

\begin{proof}
First we note that $\rm{MSE}$ and $\rm{CWD}$ are identical, except that in $\rm{CWD}$ we are minimizing with respect to the the joint distributions $\PP_{y,\hat y|x}$, while in $\rm{MSE}$ $\PP_{y,\hat y|x}$ is given and satisfies the $\rm{CWD}$ optimization constraints. Thus $\rm{MSE}\geq\rm{CWD}$. 

Next, we turn to prove that $\rm{CWD}\geq\rm{WD}$.
 $\rm{CWD}$ involves an infinite number of optimization problems parameterized by $x$. We denote the optimal solution for each $x$ by $\PP^{\star}_{y,\hat y|x}$. Then 
\begin{eqnarray}
{\rm{CWD}}
&=&\int\left(\iint \|y-\hat{y}\|^2\PP^{\star}_{y,\hat y|x}(y,\hat y) dyd\hat y)\right)\mathbb{Q}_{x}(x) dx\nonumber \\
&=&\iiint\|y-\hat{y}\|^2\PP^{\star}_{y,\hat y|x}(y,\hat y)\mathbb{Q}_{x}(x) dxdyd\hat y
\end{eqnarray} 
We define
\begin{align}
    \PP^{\star}_{y,\hat y}=\int\PP^{\star}_{y,\hat y|x}\mathbb{Q}_xdx
\end{align}
and note that it is a feasible solution for WD, as it satisfies:
\begin{eqnarray*}
 \int\PP^{\star}_{y,\hat y}dy &=& \iint\PP^{\star}_{y,\hat y|x}(y,\hat y)\mathbb{Q}_{x}(x)dxdy \\
 &=& \int {\mathbb{Q}}_{\hat y,x} dx = \mathbb{Q}_{\hat y}
\end{eqnarray*}
and the same for $\int\PP^{\star}_{y,\hat y}d\hat y$.
Thus WD satisfies:
\begin{eqnarray*}
    \rm{WD} &=&
     \min_{\PP_{y,\hat y}\in\Pi({\mathbb{Q}}_y;{\mathbb{Q}}_{\hat y})} \mathbb{E}\left[\|y-\hat y\|^2\right] \\
     & \leq & \iint \|y-\hat{y}\|^2 \PP^{\star}_{y,\hat y}(y, \hat{y}) dyd\hat y \\
     & = & \iiint\|y-\hat{y}\|^2\PP^{\star}_{y,\hat y|x}(y,\hat y)\mathbb{Q}_{x}(x) dxdyd\hat y \\
     & = & \rm{CWD},
     \end{eqnarray*}
completing the proof.     

\end{proof}

Lemma \ref{wd_relations} emphasizes the advantages of using CWD over its competitors:
\begin{itemize}
    \item CWD is better than MSE in comparing the stochastic diversity within the distributions of $y$ and $\hat y$. Two models may be identical with zero CWD but have a large MSE due to high variability. MSE is useful when both $y$ and $\hat y$ are deterministic functions of $x$, but is otherwise misleading. 
    \item CWD is better than WD in evaluating the conditioning. Two models may have zero WD but be very different. For example, one of the models can completely ignore the conditioning and just shuffle the outputs. On the other hand, if CWD is zero, then so is WD and the models are identical.
\end{itemize}

\section{Conditional Frechet Inception Distance}\label{sec_cf}

In practice, generative models are typically evaluated using a special case of WD known as FID. FID is a successful approximation that tradeoffs sample and computational complexities with expressive power. First, it restricts the attention to multivariate normal marginals which involve only second order statistics. Second, it does not work directly with $y$ and $\hat y$, but with their embeddings which are hopefully of lower dimension and ``more Gaussian''. These nonlinear embeddings are very important and will be examined in the experimental section, but, in this section, we assume that the data was previously embedded and work directly with $\{x,y,\hat y\}$. 

Consider two multivariate Gaussian distributions 
\begin{eqnarray}
{\mathbb{Q}}_{y}&=&
{\mathcal{N}}\left( m_{y} , C_{yy}\right)\\
{\mathbb{Q}}_{\hat y}&=&
{\mathcal{N}}\left( m_{\hat y} , C_{\hat y\hat y}\right)
\end{eqnarray}
After some algebraic manipulations, plugging these into WD yields  \cite{ding2020subsampling,heusel2017gans}
\begin{align}\label{eq_fid}
 {\rm{FID}}=&\|m_{y}-m_{\hat y}\|^2\\
 &+{\rm{Tr}}\left(C_{yy}+C_{\hat y\hat y} -2
 \left({C_{yy}}^{\frac{1}{2}}C_{\hat y \hat y}{C_{yy}}^{\frac{1}{2}}\right)^{\frac{1}{2}}
 \right)\nonumber
\end{align}
where $C^\frac 12$ denotes the squared root of a positive definite matrix $C\succ 0$ as defined by its eigenvalue decomposition. If the matrix is singular it is common practice to use a small diagonal regularization $(C+\epsilon I)^{\frac 12}$.


Moving on to the conditional setting in Fig. \ref{fig:cond_graphical_model} we define
\begin{eqnarray}
{\mathbb{Q}}_{y|x}(x)&=&
{\mathcal{N}}\left( m_{y|x} , C_{yy|x}\right)\\
{\mathbb{Q}}_{\hat y|x}(x)&=&
{\mathcal{N}}\left( m_{\hat y|x} , C_{\hat y\hat y|x}\right)
\end{eqnarray}
Both are conditioned on the same $x$ distributed as
\begin{eqnarray}
{\mathbb{Q}}_{x}&=&
{\mathcal{N}}\left( m_{x} , C_{xx}\right).
\end{eqnarray}
The conditional moments satisfy the identities \cite{kay1993fundamentals}:
\begin{eqnarray}\label{mmse1}
m_{y|x}&=&m_y+C_{yx}C_{xx}^{-1}(x-m_x)\nonumber\\
C_{yy|x}&=& C_{yy}-C_{yx}C_{xx}^{-1}C_{xy}
\end{eqnarray}
\begin{eqnarray}\label{mmse2}
m_{\hat y|x}&=&m_{\hat y}+C_{{\hat y}x}C_{xx}^{-1}(x-m_x)\nonumber\\
C_{\hat y\hat y|x}&=&C_{\hat y\hat y} - C_{\hat yx}C_{xx}^{-1}C_{x\hat y}
\end{eqnarray}
Note that this simplistic model leads to means that depend linearly on the inputs $x$, but the conditional covariances are constant and independent of $x$ (see future work).
Similarly to (\ref{eq_fid}), CFID is defined in the next lemma.
\begin{lemma}\label{CFID-mutual-Gaussian}
The conditional FID is given by:
\end{lemma}

\begin{align}\label{general_cfid}
{\rm{CFID}} &= \mathbb{E}_x\big[\|m_{y|x}-m_{\hat y|x}\|^2\\
&\hspace{-6mm}+
 {\rm{Tr}}\left[C_{yy|x}+C_{\hat y\hat y|x} -2 \left(({C_{yy|x}}^{\frac{1}{2}})C_{\hat y \hat y|x}({C_{yy|x}}^{\frac{1}{2}})\right)^{\frac{1}{2}}\right]\Bigg]\nonumber\\
&\hspace{-10mm}= \|m_y-m_{\hat y}\|^2+
 {\rm{Tr}}\left[\nonumber(C_{yx}-C_{\hat y x})C_{xx}^{-1}(C_{xy}-C_{x\hat y})\right]\nonumber\\
 &\hspace{-6mm}+{\rm{Tr}}\left[C_{yy|x}+C_{y\hat y|x} -2 \left(({C_{yy|x}}^{\frac{1}{2}})C_{\hat y \hat y|x}({C_{yy|x}}^{\frac{1}{2}})\right)^{\frac{1}{2}}\right] \nonumber
\end{align}

\begin{proof}
The proof is based on plugging the conditional mean and covariances and simple algebraic manipulations. Full details are in the appendix.
\end{proof}
As expected, it is easy to see that CFID is only a function of the second order statistics. It is zero if and only if $m_y=m_{\hat y}$ as well as $C_{yy}=C_{\hat y\hat y}$, $C_{yx}=C_{\hat yx}$.

\subsection{Sample Distances}

In practice, we do not have access to the true distribution functions but use a dataset ${\mathcal{D}}=\{x_i,y_i, \hat y_i\}_{i=1}^n$ where $x_i$ is a conditioning feature vector, $y_i$ is a true output and $\hat y_i$ is a generated output. Estimating CFID given such a dataset is straight forward. Each of the moments in (\ref{mmse1})-(\ref{mmse2}) is simply replaced by its empirical version, namely the corresponding sample means and sample covariances. 

As was shown by \cite{morozov2020self} and \cite{chong2020effectively},  estimates using finite samples are biased and the sample complexity depends on the dimension of the embedding and the values of the true covariance matrices. We provide more details on this issue in the experimental section.

\section{Conditional vs. Restricted models}
\begin{figure}
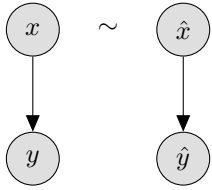

         \centering
         \tikz{
 \node[obs] (y) {$y$};%
 \node[obs,xshift=2cm] (haty) {$\hat y$};%
 \node[obs,above=of y] (x) {$x$}; %
 \node[obs,above=of haty] (hatx) {$\hat x$}; %
 \node[const,above=of y,xshift=1cm,yshift=.3cm] (eq) {$\sim$}; %
 \edge {x} {y};
 \edge {hatx} {haty} }

        \caption{Graphical Model for restricted model.}
    \label{fig:rest_graphical_model}
\end{figure}
In this section, we present an alternative approach for evaluating conditional models. It behaves somewhere ``between'' the marginal and the conditional approaches, and sheds more light on their differences. In particular, the  restricted approach is based on comparing the joint distributions of $\{y,x\}$ vs. $\{\hat y,\hat x\}$.

First, we define the WD between the two joint distributions with different inputs denoted by $\{y,x\}$ and $\{\hat y,\hat x\}$:
\begin{equation}
    {\rm{JWD}}={\mathcal{WD}}({{\mathbb{Q}}_{y,x}};{\mathbb{Q}}_{\hat y, \hat x})
\end{equation}
In the multivariate normal case, this is exactly the definition of FJD in \cite{devries2019evaluation}.

In the context of conditional models, it makes sense to restrict the attention to the case in which $x\sim\hat x$, i.e. ${\mathbb{Q}}_x={\mathbb{Q}}_{\hat x}$. This yields the Restricted Wasserstein Distance (RWD)
\begin{equation}
    {\rm{RWD}}={\mathcal{WD}}({{\mathbb{Q}}_{y,x}};{\mathbb{Q}}_{\hat y, x})
\end{equation}
Intuitively, JWD allows different marginals for $x$ and $\hat x$. RWD assumes these inputs are independent realizations from identical marginal distributions, i.e., ${\mathbb{Q}}_x={\mathbb{Q}}_{\hat x}$. Finally, CWD assumes that the realizations themselves are identical, i.e., $x=\hat x$. The probabilstic model is illustrated in Fig. \ref{fig:rest_graphical_model}.

The next lemma shows that RWD is generally somewhere ``between'' WD and CWD. It can approach CWD if we put more emphasis on the inputs. For this purpose, we define two triplets of random variables which are identical up to a scaling factor on $x$:
\begin{align}
    {\mathcal{D}}&=\{x,y,\hat y\}\nonumber\\
    {\mathcal{D}}_\alpha&=\{\alpha x,y,\hat y\}
\end{align}
where $\alpha$ is a deterministic scaling factor. Let CWD and RWD denote the metrics associated with ${\mathcal{D}}$, and CWD$_\alpha$ and RWD$_\alpha$ denote the metrics associated with the scaled ${\mathcal{D}}_\alpha$. 

\begin{lemma}\label{rwd_vs_cwd_mwd}
The Wasserstein distances satisfy
   \begin{equation}
       {\rm{CWD}} \geq \rm{RWD} \geq {\rm{WD}}.
   \end{equation}
   Furthermore,
 \begin{align}
    & {\rm{CWD}}_\alpha={\rm{CWD}}\qquad \forall \quad \alpha>0\nonumber\\
    & {\rm{RWD}}_0={\rm{WD}}  \nonumber\\
    & {\rm{RWD}}_\alpha\stackrel{\alpha\rightarrow\infty}{\rightarrow}{\rm{CWD}}.  
\end{align}  
   
\end{lemma}

\begin{proof}
The full proof is provided in the appendix. The main idea is to introduce an additional metric denoted by RWD3:
    \begin{equation}
    {\rm RWD3} =
    \left\{
    \begin{array}{ll} \min_{\PP_{x, y,\hat y}}  & \mathbb{E}\left[\|y-\hat y\|^2\right]   \\
   {\rm{s.t.}} & \PP_{x, y,\hat y}\in\Pi({\mathbb{Q}}_{x,y};{\mathbb{Q}}_{x, \hat y})
    \end{array}
    \right.
    \end{equation}
 As its name suggests, RWD3 is based on a minimization with respect to $\PP_{y,\hat y,x}$, whereas RWD relies on a minimization with respect to $\PP_{y,\hat y,x,\hat x}$. This makes RWD3 more similar to CWD. The proof follows by showing that
 \begin{align}\label{four_ineq}
      {\rm{CWD}} = {\rm{RWD3}}\geq \rm{RWD} \geq {\rm{WD}}
 \end{align}
 due to their feasible sets and objectives.

The second part of the proof shows that the first inequality in (\ref{four_ineq}) is actually tight if we replace RWD by RWD$_\alpha$ with $\alpha\rightarrow\infty$. For this purpose, note that
    \begin{equation}
    {\rm RWD}_\alpha =
    \left\{
    \begin{array}{ll} \min_{\PP_{x, \hat x, y,\hat y}}  & \mathbb{E}\left[\|y-\hat y\|^2+ \alpha^2\|x-\hat x\|^2\right]   \\
   {\rm{s.t.}} & \PP_{x, \hat x, y,\hat y}\in\Pi({\mathbb{Q}}_{x,y};{\mathbb{Q}}_{\hat x, \hat y})
    \end{array}
    \right.
    \end{equation}
For unbounded $\alpha$, RWD$_\alpha$ is finite only if $x=\hat x$. Thus, its optimal solution satisfies
\begin{equation}
    {\PP^{\star}}_{x,\hat x, y,\hat y} = \PP^\star_{x, y,\hat y}\delta_{x-\hat x}
\end{equation}
where $\delta_{x-\hat x}$ is the delta of Dirac. This is exactly the definition of RWD3 which is identical to CWD.

Next, for $\alpha=0$, the joint distributions $\mathbb{Q}_{\alpha x,y}$ and $\mathbb{Q}_{\alpha\hat x,\hat y}$ are reduced to the marginal distributions $\mathbb{Q}_{y}$ and $\mathbb{Q}_{\hat y}$, and the objective is reduced to $\norm{y - \hat y}^2$. Together, $\rm RWD_{\alpha=0}=\rm WD$.

Finally, by changing variables $x'=\alpha x$, it is easy to show that ${\rm CWD}_{\alpha}$ is invariant to $\alpha$.

\end{proof}

To conclude this section, CWD is invariant to scaling of $x$. This agrees with the well known property that correlation coefficients are invariant to scaling of their random variables. On the other hand, depending on the scaling, RWD changes between WD and CWD. This behaviour is illustrated in 
Fig. \ref{invar_alpha} for the special case of Gausssian distributions in which RWD$_\alpha$ particularizes to RFID$_\alpha$ and can be easily computed.
\begin{figure}
\centering
\includegraphics[width=8cm]{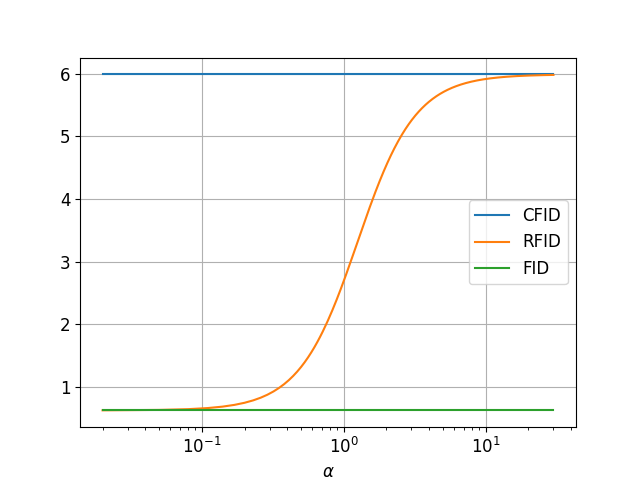}
\caption{FID, RFID$_\alpha$ and CFID as a function of the scaling of the input $\alpha$. Here we use a triplet $x,y,\hat y$ of scalars with zero means and a randomly generated covariance. As expected, FID and CFID are invariant to the scaling while RFID$_\alpha$ converges to FID for small values of $\alpha$ and to CFID for large values of $\alpha$.}
\label{invar_alpha}
\end{figure}

\section{Experiments}\label{sec_exper}
\begin{figure}
\centering{\includegraphics[width=8cm]{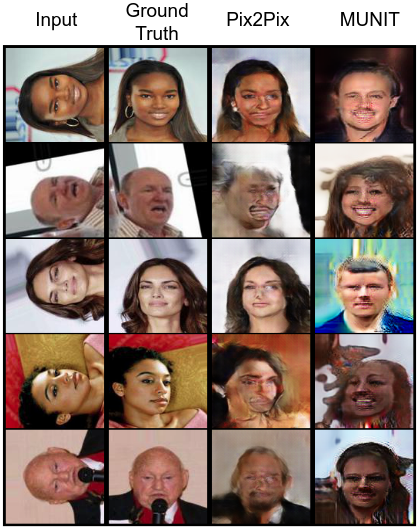}}
\caption{Visual results in CelebA. From left to right: input, ground-truth, {\sc Pix2Pix}, {\sc MUNIT}. {\sc MUNIT} generates more realistic images and the marginal distribution of its outputs is closer to the ground truth images. But it completely ignores the input and its conditional distribution is very far from the true one.
}\label{fig_celeba}
\end{figure}

\begin{table}[h]
    \centering
    \begin{sc}
    \begin{tabular}{llll}
    \toprule
    Model & Unpaired & Diverse \\
    \midrule
Pix2Pix &  no & no \\
        MUNIT &  yes & yes \\       Aug$80\%$ & no & yes \\
    Aug$10\%$Fixed & partially & no \\
    Aug$10\%$ & partially & yes \\
    \bottomrule
    \end{tabular}
    \end{sc}
    \caption{Different generative models and their properties. }
    \label{model_table}
\end{table}

In this section, we evaluate different generative models using the various metrics via numerical experiments. In all the experiments, we consider FID, CFID, and RFID. For comparison, we also provide MSE results using the same embedding as the FID metrics. Finally, as common in the image processing literature, we also report PSNR results (on the raw pixels without any embedding). 

For convenience, we list the different generative models and their properties. In Table \ref{model_table}, Paired models are trained on an aligned dataset of inputs and outputs $\{x_i,y_i\}_{i=1}^N$ where $x_i$ is paired with $y_i$ for the same index $i$. Unpaired models use two decoupled datasets for the inputs $\{x_i\}_{i=1}^{N_x}$ and outputs $\{y_i\}_{i=1}^{N_y}$. Diverse models generate multiple outputs for the same input. 

All models were trained using PyTorch \cite{zhu2017toward}, with the default parameters. More details are provided below in the different subsections.

 \begin{figure}
\centering{\includegraphics[width=8cm]{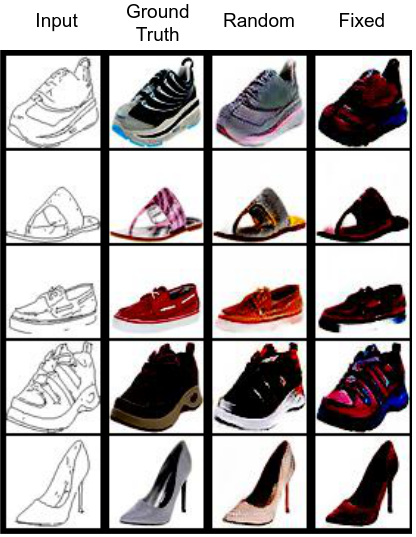}}
\caption{Visual results on edges2hoes. From left to right: input, ground-truth, Augmented CycleGAN with random latent input, Augmented CycleGAN with fixed latent input. As the problem is ill conditioned, both models fail to colorize the edges with the correct colors. However, the random model outputs more reasonable images while the fixed model outputs only black-red shoes that do not capture the full distribution. }\label{fig_Edges2Shoes}
\end{figure}

\subsection{Paired models}
Our first set of experiments considered intensive simulations comparing paired many conditional generators (in addition to those detailed above) in different datasets. We compared various metrics including FID, CFID, IS, LPIPS, SSIM, PSNR and related metrics on downstream tasks. We expected that CFID would provide a better evaluation than FID, but both metrics typically behaved similarly. 
Apparently, all the models succeeded in “squeezing out” the information available in their paired dataset, and mostly differed in their output distributions as measured by FID. This led us to consider more modern unpaired models in the next experiments. There, the main challenge is learning the relations between the inputs and outputs.

\subsection{Paired vs unpaired models}
Our second set of experiments compared a paired {\sc Pix2Pix} model from \cite{isola2017image} to an unpaired {\sc MUNIT} model from \cite{huang2018munit}. We used a subset of the the CelebA \cite{liu2015deep} dataset with $n=80000$ samples, divided into a training set of size $30000$ and a test set of size $50000$. For the FID embedding we used the last max-pool InceptionV3 layer \cite{szegedy2016rethinking}.
 The results of the experiment are provided in Table \ref{table_metrics1}. Fig. \ref{fig_celeba} shows a few images that were generated by the models.
The visual results show that {\sc MUNIT} generates slightly better looking images, but these are not related to their inputs. Thus, while {\sc MUNIT} is better in terms of FID, it is much worse in terms of PSNR. CFID identifies this failure and grades {\sc Pix2Pix} better. In this experiment, RFID behaves similarly to CFID.

\begin{table}[t]
\caption{Performance on CelebA rotation task. {\sc Pix2Pix} is paired whereas {\sc MUNIT} is unpaired.  While {\sc MUNIT} learns a better marginal distribution, its conditioning on the input is worse. Thus, {\sc MUNIT} achieves better FID whereas {\sc Pix2Pix} achieves better CFID, RFID, MSE and PSNR.}
\label{sample-table}
\vskip 0.15in
\begin{center}
\begin{scriptsize}
\begin{sc}
\begin{tabular}{lccccr}
\toprule
Model & FID & RFID & CFID & MSE & PSNR \\
\midrule
Pix2Pix    & 41.23 & \textbf{48.55} & \textbf{62.39} & \textbf{171} & \textbf{17.78}  \\
MUNIT   & \textbf{37.02} & 58.23 & 91.82 & 240 & 8.54   \\
\bottomrule
\end{tabular}
\end{sc}
\end{scriptsize}
\end{center}
\vskip -0.1in
\label{table_metrics1}

\end{table}
 \subsection{Diverse vs. fixed model}

The previous experiment shows that CFID punishes models that fail to learn the correct conditioning. PSNR also has this property, and therefore the third set of experiments considers the main advantage of CFID over PSNR, namely its ability to evaluate diversity within the distribution.  For this purpose, we used the Augmented CycleGAN model from \cite{almahairi2018augmented}. It allows different percentages of paired vs unpaired training sets. It can also be configured to generate a fixed output or a diverse set of outputs (by either using a fixed latent value for the entire test set or by using randomized latent values, respectively). We use the Edges2Shoes dataset \cite{isola2017image}. The original dataset has only $300$ examples in the test set, which is insufficient for estimate the metrics accurately. Therefore, we re-divided the data to a training set of size $35000$ and a test set of size $10000$. 

The numerical results are provided in Table \ref{table_metrics2} and Fig. \ref{fig_Edges2Shoes} shows a few representative images that were generated by the models. 
The visual results show that while both {\sc Aug$10\%$ } and {\sc Aug$10\%$Fixed} (that is, Augmented CycleGan with $10\%$ paired data with random and fixed latent input respectively) fail to colorize the edges with the correct colors, the random model outputs more reasonable images. The fixed model outputs only black-red shoes that do not capture the full distribution. Indeed, the numerical results also show that the PSNR of the fixed model is better but its FID is much worse. As promised, CFID captures both effects and grades the random model better. To complete the picture we also compare to a stronger model which is both $80\%$ paired and diverse -  {\sc Aug$80\%$}. It is  better than all its competitors. However, the gains vs the  {\sc Aug$10\%$} model are quite small, probably because of the locality bias of the transformation \cite{richardson2020surprising}. In all the models, RFID is consistent with CFID.

\begin{table}[t]
\caption{Performance on Edges2Shoes.  {\sc Aug$80\%$} is both paired and diverse and achieves the best CFID/RFID scores. {\sc Aug$10\%$} is similar but slightly behind in terms of CFID. Ten-fold cross validation suggests that the two models have almost identical FID, but that the CFID gap is more pronounced (the gap is around 4 standard deviations). Finally, the non-diverse model {\sc Aug$10\%$Fixed} is optimal in terms of MSE/PSNR but yields a poor CFID score.}
\label{sample-table}
\vskip 0.15in
\begin{center}
\begin{scriptsize}
\begin{sc}
\begin{tabular}{lccccr}
\toprule
Model & FID & RFID & CFID & MSE & PSNR \\
\midrule
Aug10\%              & \textbf{9.90} & 13.80 & 26.57  & 167 & 15.74 \\
Aug10\% fixed      & 31.56 & 35.88 & 45.81 & \textbf{153} & \textbf{17.78}  \\
Aug80\%              & 9.91 & \textbf{13.76} & \textbf{26.11} & 164 & 15.78  \\

\bottomrule
\end{tabular}
\end{sc}
\end{scriptsize}
\end{center}
\vskip -0.1in
\label{table_metrics2}

\end{table}

\subsection{Different Embeddings} 
The performance of CFID clearly depends on the chosen embedding of the images. Recently, it was shown that different embeddings can result in different FID rankings \cite{morozov2020self}. Self supervised embeddings were recommended over the standard InceptionV3 embedding which was trained for image classification. Thus, we repeat the Edges2Shoes experiment with the SwAV embedding \cite{caron2020unsupervised}. Table \ref{table_metrics3} shows that although the values of the metrics are different, they are consistent with the InceptionV3 embedding and we see the same trend. The only difference is that now the RFID of the $10\%$ supervised model is slightly better than the $80\%$ model, demonstrating the higher sensitivity of CFID.

\begin{table}[t]
\caption{Performance on Edges2Shoes with SwAV embedding. 
The results are  similar to Table \ref{table_metrics2} with the InceptionV3 embedding. The values are different but with the same trends (and slight insignificant differences).}
\label{sample-table}
\vskip 0.15in
\begin{center}
\begin{scriptsize}
\begin{sc}
\begin{tabular}{lccccr}
\toprule
Model & FID & RFID & CFID & MSE & PSNR \\
\midrule
Aug 10\%              & \textbf{0.89} & \textbf{1.13} & 1.95  & 8.69 & 15.74 \\
Aug 10\% fixed      & 2.54 & 2.79 & 3.03 & \textbf{7.06} & \textbf{17.78}  \\
Aug 80\%              & 0.91 & 1.15 & \textbf{1.91} & 8.41 & 15.78  \\

\bottomrule
\end{tabular}
\end{sc}
\end{scriptsize}
\end{center}
\vskip -0.1in
\label{table_metrics3}

\end{table}

\subsection{Different Embeddings for Input and Output} 

CFID involves inputs $x$ and outputs $\{y,\hat y\}$ all of which need to be cleverly embedded for meaningful metrics. In the previous sections, we used the same embeddings for the inputs and the outputs. However, it should be noticed that the inputs are not necessarily in the same domain of the output. In some cases, the inputs are even natural language whereas the outputs are images \cite{zhang2018stackgan++}. Thus, different embedding may be needed for the input and the outputs. The Edges2Shoes dataset is also an example for this scenario, since the input is very far from a natural image and a model that was trained on natural image dataset (as InceptionV3 and SwAV) may be inappropriate. Thus we repeat the Edges2Shoes experiment again, now without an embedding at all on the input and with an InceptionV3 embedding for the output. In order to achieve a reasonable sample complexity, we downscale the input image to size 32x32. Table \ref{table_metrics4} shows that values of CFID and RFID are slightly changed (FID and MSE are clearly unchanged). Nonetheless, the trends are the same as before. Here too, the RFID of the $10\%$ supervised model is slightly better than the $80\%$ model.

\begin{table}[t]
\caption{Performance on Edges2Shoes with a raw downscaling embedding of the input images. 
The results are  similar to Table \ref{table_metrics2} with the InceptionV3 embedding. The values are different but with the same trends (and slight insignificant differences).}
\label{sample-table}
\vskip 0.15in
\begin{center}
\begin{scriptsize}
\begin{sc}
\begin{tabular}{lccccr}
\toprule
Model & FID & RFID & CFID & MSE & PSNR \\
\midrule
Aug 10\%              & \textbf{9.90} & \textbf{15.15} & 29.7  & 167 & 15.74 \\
Aug 10 fixed\%       & 31.56 & 35.67 & 46.20 & \textbf{153} & \textbf{17.78}  \\
Aug 80\%              & 9.91 & 15.17 & \textbf{29.10} & 164 & 15.78  \\

\bottomrule
\end{tabular}
\end{sc}
\end{scriptsize}
\end{center}
\vskip -0.1in
\label{table_metrics4}

\end{table}

\subsection{Sample Complexity}
In the last experiment, we explore the sample complexity of CFID in practice. Following \cite{morozov2020self} we examine the dependence of the estimated values of RFID and CFID  on the number of samples. Figure \ref{sample_complexity} shows the estimated values of the FID, RFID and CFID as a function of the number of samples, normalized on the values that were estimated using the entire test data (50k samples). We use the {\sc Pix2Pix} model on the rotation task. As expected, the results show that all the metrics behave similarly. As noted in \cite{morozov2020self} for FID, all the metrics overestimate their population values. CFID and RFID both require more samples than FID, and CFID is a bit more efficient.

\begin{figure}
\centering{\includegraphics[width=8cm]{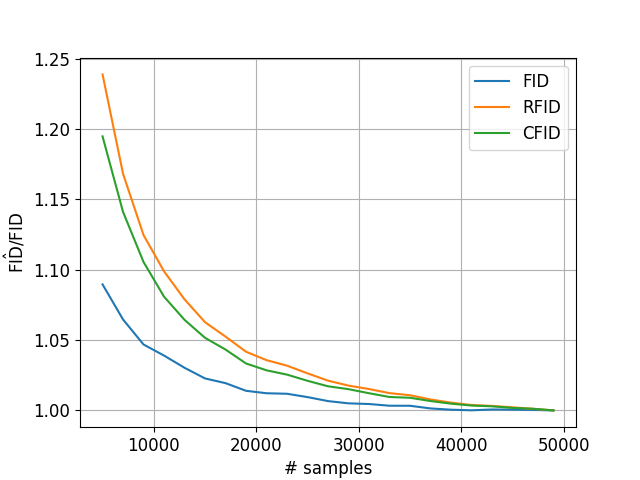}}
\caption{Normalized estimate values FID, RFID and CFID as a function of the number of samples. }
\label{sample_complexity}
\end{figure}

\section{Discussion}

In this paper, we considered conditional versions of the popular Wasserstein and Frechet distances. We introduced novel metrics and addressed their use in evaluating conditional generative models. In these settings, we showed that the conditional metrics outperform both the classical MSE metric and their corresponding unconditional metrics. In particular, a good conditional score signals that the model succeeded in its two tasks: generating realistic outputs that correspond to the correct input; and generating a diverse set of samples that fully capture the underlying output distribution.


There are many directions for future work on conditional metrics. Even within the Gaussian framework, stronger FID metrics can allow more expressive models, e.g. non-constant conditional covariances that depend on the inputs. For more realistic measures, non-Gaussian distributions should be considered as in \cite{assa2018wasserstein,dimitrakopoulos2020wind}.
Theory can focus on sample complexity analysis of the different metrics and its dependence on the different parameters and embeddings.  


\bibliographystyle{plain}
\bibliography{main}

\onecolumn
\appendix
\subsection{Proof of Lemma 2}

This proof is quite technical and straight forward. We begin with the means:
$\mathbb{E}_x[\|m_{y|x}-m_{\hat y|x}\|^2]$=
\begin{align}
=\mathbb{E}_x\left[{\rm{Tr}}\left(m_{y|x}m_{y|x}^{T}-2m_{y|x}m_{\hat{y}|x}^{T}+m_{\widetilde{y}|x}m_{\widetilde{y}|x}^{T}\right)\right]=
\end{align}
For the term $\mathbb{E}_x\left[{\rm{Tr}}\left(m_{y|x}m_{y|x}^T\right)\right]:$
\begin{align}
&\mathbb{E}_x\left[{\rm{Tr}}\left((m_y+C_{yx}C^{-1}_{xx}(x-m_x))((x-m_x)^T C_{xx}^{-1}C_{xy}+m_y^T)\right)\right]=\nonumber\\
&={\rm{Tr}}\left(0+m_y\cdot m_{\hat y}^T+C_{yx}C^{-1}_{xx}C_{xx}C^{-1}_{xx}C_{xy}+0\right)=\nonumber\\
&={\rm{Tr}}\left(m_y\cdot m_{\hat y}^T+C_{yx}C^{-1}_{xx}C_{xy}\right)
\end{align}
For the term $-2\mathbb{E}_x\left[{\rm{Tr}}\left(m_{y|x}m_{\hat{y}|x}^{t}\right)\right]:$
\begin{align}
&-2\mathbb{E}_x\left[{\rm{Tr}}\left((m_{\hat y}+C_{\hat y x}C^{-1}_{xx}(x-m_x))((x-m_x)^T C_{xx}^{-1}C_{x\hat y}+m_{\hat y}^T)\right)\right]=\nonumber\\
&={\rm{Tr}}\left(-2(0+m_{\hat y}\cdot m_{\hat y}^T+C_{yx}C^{-1}_{xx}C_{xx}C^{-1}_{xx}C_{x\hat{y}}+0)\right)=\nonumber\\
&={\rm{Tr}}\left(-2(m_y\cdot m_{\hat y}^T+C_{yx}C^{-1}_{xx}C_{x\hat y})\right)
\end{align}
For the term $\mathbb{E}_x\left[{\rm{Tr}}\left(m_{\hat y|x}m_{\hat y|x}^T\right)\right]:$
\begin{align}
&\mathbb{E}_x\left[{\rm{Tr}}\left((m_{\hat y}+C_{\hat{y}x}C^{-1}_{xx}(x-m_x))((x-m_x)^T C_{xx}^{-1}C_{x\hat{y}}+m_y^T)\right)\right]=\nonumber\\
&={\rm{Tr}}\left(0+m_{\hat{y}}\cdot m_{\hat y}^T+C_{\hat{y}x}C^{-1}_{xx}C_{xx}C^{-1}_{xx}C_{x{\hat y}}+0\right)=\nonumber\\
&={\rm{Tr}}\left(m_{\hat{y}}m_{\hat y}^T+C_{\hat{y}x}C^{-1}_{xx}C_{x\hat{y}}\right)
\end{align}
Thus 
\begin{align}
&\mathbb{E}_x\left[{\rm{Tr}}\left((m_y+C_{yx}C^{-1}_{xx}(x-m_x))((x-m_x)^T C_{xx}^{-1}C_{xy}+m_y^T)\right)\right.=\nonumber\\
&\mathbb{E}_x\left[{\rm{Tr}}\left(\left(m_{{y}}m_{y}^T-2\cdot m_{{y}}m_{\hat y}^T+m_{\hat{y}}m_{\hat y}^T\right)\right.\right.+\nonumber\\
& +\left.\left.\left(C_{{y}x}C^{-1}_{xx}C_{x\hat{y}}-2\cdot C_{{y}x}C^{-1}_{xx}C_{x\hat{y}}+C_{\hat{y}x}C^{-1}_{xx}C_{x\hat{y}}\right)\right)\right]=\nonumber\\
&\mathbb{E}_x\left[\|m_{{y}}-m_{\hat y}\|^2+\right.\nonumber\\
&+\left.{\rm{Tr}}\left(\left(C_{{y}x}\left(C^{-1}_{xx}\left(C_{x{y}}-C_{x\hat{y}}\right)\right)\right)+\left(\left(\left(C_{{y}x}-C_{\hat{y}x}\right)C^{-1}_{xx}\right)C_{x\hat{y}}\right)\right)\right]=\nonumber\\
&\mathbb{E}_x\left[\|m_{{y}}-m_{\hat y}\|^2+\right.\nonumber\\
&+\left.{\rm{Tr}}\left(\left(\left(C_{{y}x}-C_{\hat{y}x}\right)C^{-1}_{xx}\right)C_{{x}y}-\left(\left(C_{{y}x}-C_{\hat{y}x}\right)C^{-1}_{xx}\right)C_{x\hat{y}}\right)\right]=\nonumber\\
&\|m_{{y}}-m_{\hat y}\|^2+{\rm{Tr}}\left(\left(C_{{y}x}-C_{\hat{y}x}\right)C^{-1}_{xx}\left(C_{{x}y}-C_{x\hat{y}}\right)\right)
\end{align}
Now back to CFID, we conclude that:
\begin{align}
{\rm{CFID}} =&\|m_y-m_{\hat y}\|^2+\\
+&{\rm{Tr}}\left[\nonumber(C_{yx}-C_{\hat y x})C_{xx}^{-1}(C_{xy}-C_{x\hat y})\right]\\
+&{\rm{Tr}}\left[C_{yy|x}+C_{y\hat y|x} -2 \left({C_{yy|x}}^{\frac{1}{2}}C_{\hat y \hat y|x}{C_{yy|x}}^{\frac{1}{2}}\right)^{\frac{1}{2}}\right] \nonumber
\end{align}

\subsection{}{Proof of Lemma 3}

For convenience, we begin by recalling the definitions of the different distances: 
\begin{equation}
    {\rm{CWD}} = \int {\rm{CWD}}(x) {\mathbb{Q}}_xdx
\end{equation}
\begin{equation}
    {\rm{CWD}}(x) = \left\{\begin{array}{ll}
        \min_{\PP_{y,\hat y|x}} & \iint \|y-\hat y\|^2 \PP_{y,\hat y|x} dy d\hat y\\
        {\rm{s.t.}} & \int \PP_{y,\hat y|x}dy = {\mathbb{Q}}_{\hat y|x}\\
       & \int \PP_{y,\hat y|x}d\hat y = {\mathbb{Q}}_{y|x}\\
    \end{array}\right.
\end{equation}
\begin{equation}
    {\rm{RWD3}} = \left\{\begin{array}{ll}
        \min_{\PP_{y,\hat y,x}} & \iiint \|y-\hat y\|^2 \PP_{y,\hat y,x} dy d\hat y dx\\
        {\rm{s.t.}} & \int \PP_{y,\hat y,x}dy = {\mathbb{Q}}_{\hat y,x}\\
       & \int \PP_{y,\hat y,x}d\hat y = {\mathbb{Q}}_{y,x}\\
    \end{array}\right.
\end{equation}
\begin{equation}
    {\rm{RWD}} = \left\{\begin{array}{ll}
        \min_{\PP_{y,\hat y,x,\hat x}} & \iiiint \left\|\begin{array}{c}
             y - \hat y \\
             x - \hat x
        \end{array}\right\|^2\PP_{y,\hat y,x,\hat x}dy d\hat y dx d\hat x \\
        {\rm{s.t.}} & \iint \PP_{y,\hat y,x,\hat x}dydx = {\mathbb{Q}}_{\hat y,x}(\hat y,\hat x)\\
       & \iint \PP_{y,\hat y,x,\hat x}d\hat yd\hat x = {\mathbb{Q}}_{y,x}(y,x)\\
    \end{array}\right.
\end{equation}

\begin{equation}
    {\rm{MWD}} = \left\{\begin{array}{ll}
        \min_{\PP_{y,\hat y}} &\iint \left\|y - \hat y\right\|^2\PP_{y,\hat y}dy d\hat y \\
        {\rm{s.t.}} & \int \PP_{y,\hat y}dy = {\mathbb{Q}}_{\hat y}\\
       &\int  \PP_{y,\hat y}d\hat y = {\mathbb{Q}}_{y}\\
    \end{array}\right.
\end{equation}

Part 1: We begin by proving 
 $${\rm{CWD}}\geq {\rm{RWD3}}$$   
CWD involves an infinite number of optimization problems parameterized by $x$. We denote the optimal solution to CWD$(x)$ by $\PP^{\star}_{y,\hat y|x}$. We then
define
\begin{eqnarray}\label{pstar}
 \PP^{\star}_{x,y,\hat y}:=\ \PP^{\star}_{y,\hat y|x}\cdot \mathbb{Q}_x
\end{eqnarray}
We prove the inequality by showing that (\ref{pstar}) is feasible for RWD3 and yields the same objective.
For this purpose, note that it is a legitimate distribution 
\begin{align}
\int\PP^{\star}_{y,\hat y,x}dyd\hat y dx=\int\left(\int\PP^{\star}_{y, \hat y|x}dyd\hat y \right)\mathbb{Q}_xdx=\int1\cdot\mathbb{Q}_xdx=1 
\end{align}
and feasible since  
\begin{align}
\int\PP^{\star}_{x,y, \hat y}d\hat y=\left(\int\PP^{\star}_{y, \hat y|x}d\hat y\right)\,\mathbb{Q}_x=\PP_{y|x}\cdot\mathbb{Q}_x=\mathbb{Q}_{y,x}
\end{align}
\begin{align}
\int\PP^{\star}_{x,y, \hat y}dy=\left(\int\PP^{\star}_{y, \hat y|x}d y\right)\,\mathbb{Q}_x=\PP_{\hat y|x}\cdot\mathbb{Q}_x=\mathbb{Q}_{\hat y,x}
\end{align}
Finally, it is easy to see that it gives the same objective value as CWD for:
\begin{align}\label{obejectives}
{\rm{CWD}}
=\int\left(\iint \|y-\hat{y}\|^2\PP^{\star}_{y,\hat y|x}(y,\hat y)\mathrm d (y,\hat y)\right)\mathbb{Q}_{x}(x)\mathrm d x\,=
\end{align}
\begin{align}
=\iiint \|y-\hat{y}\|^2\PP^{\star}_{x,y,\hat y}(x,y,\hat y)\mathrm dx,dy,d\hat y\,
\end{align}
Actually, the inequality is tight and ${\rm CWD}={\rm RWD3}$. To prove this, we will show the other direction:
\begin{equation}
    {\rm CWD}\leq{\rm RWD3}
\end{equation}
We denote the optimal solution to $\rm RWD3$ by $\PP^{\star}_{x,y,\hat y}$. We then define
\begin{eqnarray}\label{pstar2}
 \PP^{\star}_{y,\hat y|x} := \frac{\PP^{\star}_{x, y,\hat y}}{\mathbb{Q}_x}
\end{eqnarray}
We prove the inequality by showing that (\ref{pstar2}) is feasible for CWD and yields the same objective. 
For this purpose, note that it is legitimate distribution for all $x$:
\begin{equation}\label{3int}
    \int \PP^{\star}_{y,\hat y|x}dyd\hat y = \int \frac{\PP^{\star}_{x, y,\hat y}}{\mathbb{Q}_x}dyd\hat y 
\end{equation}
Because $\PP^{\star}_{x, y,\hat y}$ is a valid solution for RWD3, thus:
\begin{equation}\label{int}
    \int \PP^{\star}_{x, y,\hat y}dyd\hat y = \int \mathbb{Q}_{\hat y,x}d\hat y = \mathbb{Q}_x
\end{equation}
That is, (\ref{pstar2}) is a legitimate distribution for all $x$. 
Now we show that it also a valid solution:
\begin{equation}
    \int \PP^{\star}_{y,\hat y|x}dy = \int \frac{\PP^{\star}_{x, y,\hat y}}{\mathbb{Q}_x}dy=\frac{\mathbb{Q}_{\hat y,x}}{\mathbb{Q}_x}=\mathbb{Q}_{\hat y|x} 
\end{equation}
\begin{equation}
    \int \PP^{\star}_{y,\hat y|x}d\hat y = \int \frac{\PP^{\star}_{x, y,y}}{\mathbb{Q}_x}= \frac{\mathbb{Q}_{y,x}}{\mathbb{Q}_x}=\mathbb{Q}_{y|x} 
\end{equation}
Finally it yields the same objective value as RWD3 (as shown in (\ref{obejectives})).

\vspace{10mm} 

Part 2: Next, we continue to prove
 $${\rm{RWD3}}\geq {\rm{RWD}}$$
RWD3 and RWD are very similar. The only difference is that 
RWD3 minimizes over $\PP_{y,\hat y, x}$, whereas RWD minimizes over $\PP_{y,\hat y, x, \hat x}$. To prove the inequality, we denote the optimal solution to RWD3 by $\PP^\star_{y,\hat y, x}$ and claim that 
\begin{align}{{\PP^{\star}}}_{y,\hat y, x,\hat x} = \PP^\star_{y,\hat y, x}\delta_{x-\hat x}\end{align}
is feasible for RWD and yields the same objective. Indeed, the delta function ensures that $\int g(\hat x)\delta_{x-\hat x}d\hat x = g(x)$ for any function $g$. Therefore:

\begin{equation}
  \left\|\begin{array}{c}
             y - \hat y \\
             x - \hat x
        \end{array}\right\|^2\cdot\delta_{x-\hat{x}}=\left\{\begin{array}{ll}\|y-\hat y\|^2\ & x=\hat x \\
        0 & x\neq\hat x  \end{array}\right.
\end{equation}

        
        \begin{align}\iint\PP^{\star}_{y,\hat y, x}\delta_{x-\hat x}dydx=\mathbb{Q}_{\hat y,x}\end{align}
        \begin{align}\iint\PP^{\star}_{y,\hat y, x}\delta_{x-\hat x}d\hat yd\hat x=\mathbb{Q}_{ y,x}\end{align}
        as required.
\vspace{10mm}

Part 3: Finally, it remains to show that
   \begin{eqnarray}
    {\rm RWD}\geq {\rm{MWD}}
    \end{eqnarray}
For this purpose, note that
\begin{align}\left\|\begin{array}{c}
             y - \hat y \\
             x - \hat x
        \end{array}\right\|^2\geq\|y-\hat y\|^2\end{align}
Thus, we can omit the squared norm associated with $x-\hat x$. Next, we denote the optimal solution to RWD by $\PP^*_{y,\hat y,x,\hat x}$ and define 
\begin{align}\PP_{y,\hat y} = \iint\PP^*_{y,\hat y,x,\hat x}dxd\hat x\end{align}
This bivariate distribution is feasible for MWD and yields the same objective value as required.

Now we prove the second part of the lemma, that is:
 \begin{align}
    & {\rm{CWD}}_\alpha={\rm{CWD}}\qquad \forall \alpha>0\nonumber\\
    & {\rm{RWD}}_0={\rm{WD}}  \nonumber\\
    & {\rm{RWD}}_\alpha\stackrel{\alpha\rightarrow\infty}{\rightarrow}{\rm{CWD}}.  
\end{align}  
We prove the inequality $\rm RWD\leq RWD3$ is tight if we replace RWD by RWD$_\alpha$ with $\alpha\rightarrow\infty$.
 For this purpose, note that:
    \begin{equation}
    {\rm RWD}_\alpha =
    \left\{
    \begin{array}{ll} \min_{\PP_{x, \hat x, y,\hat y}}  & \mathbb{E}\left[\|y-\hat y\|^2+ \alpha^2\|x-\hat x\|^2\right]   \\
   {\rm{s.t.}} & \PP_{x, \hat x, y,\hat y}\in\Pi({\mathbb{Q}}_{x,y};{\mathbb{Q}}_{\hat x, \hat y})
    \end{array}
    \right.
    \end{equation}
For unbounded $\alpha$, RWD$_\alpha$ is finite only if $x=\hat x$. Thus, the optimal solution satisfies
\begin{equation}
    {\PP^{\star}}_{x,\hat x, y,\hat y} = \PP^\star_{x, y,\hat y}\delta_{x-\hat x}
\end{equation}
where $\delta_{x-\hat x}$ is the delta of Dirac. This is exactly the definition of RWD3 which is identical to CWD.

For $\alpha=0$, the joint distributions $\mathbb{Q}_{\alpha x,y}$ and $\mathbb{Q}_{\alpha\hat x,\hat y}$ are reduced to the marginal distributions $\mathbb{Q}_{y}$ and $\mathbb{Q}_{\hat y}$ and the objective is reduced to $\norm{y - \hat y}^2$. Thus, $\rm RWD_{\alpha=0}=\rm WD$.

Finally, we show that ${\rm CWD}_{\alpha}$ is invariant to $\alpha$.
Note that:
\begin{equation}
    {\rm{CWD}_{\alpha}} = \left\{\begin{array}{ll}
        \min_{\PP_{y,\hat y|\alpha x}} & \int\left(\iint \|y-\hat y\|^2 \PP_{y,\hat y|\alpha x} dy d\hat y\right)\mathbb{Q}_{\alpha x}dx\\
        {\rm{s.t.}} & \int \PP_{y,\hat y|\alpha x}dy = {\mathbb{Q}}_{\hat y|\alpha x}\\
       & \int \PP_{y,\hat y|\alpha x}d\hat y = {\mathbb{Q}}_{y|\alpha x}\\
    \end{array}\right.
\end{equation}
Changing variables in the integral over $x$: $x'=\alpha x$ gives the same value as the original CWD.

\end{document}